\documentclass{article}

\usepackage{arxiv}

\usepackage[utf8]{inputenc} 
\usepackage[T1]{fontenc}    
\usepackage{hyperref}       
\usepackage{url}            
\usepackage{booktabs}       
\usepackage{amsfonts}       
\usepackage{nicefrac}       
\usepackage{microtype}      
\usepackage{lipsum}
\usepackage{graphicx}

\usepackage[numbers]{natbib}
\usepackage{microtype}      
\usepackage{xcolor}         
\usepackage{paralist}       
\usepackage{multirow}

\usepackage{amsthm}
\usepackage{verbatim}
\usepackage{bm}
\usepackage{bbm}
\usepackage{mathtools}

\newtheorem{theorem}{Theorem}
\newtheorem{definition}{Definition}

\usepackage{siunitx}
\newcommand{\R}{\mathbb{R}}

\DeclarePairedDelimiter{\abs}{\lvert}{\rvert}

\graphicspath{ {./figures/} }

\title{Diffusion Curvature for Estimating Local Curvature in High Dimensional Data}

\author{Dhananjay Bhaskar\textsuperscript{1\textdagger}, Kincaid MacDonald\textsuperscript{2\textdagger}, Oluwadamilola Fasina\textsuperscript{3}, Dawson Thomas\textsuperscript{2,4},\\ \textbf{Bastian Rieck}\textsuperscript{5}, \textbf{Ian Adelstein}\textsuperscript{2 *}\& \textbf{Smita Krishnaswamy}\textsuperscript{1,3,6,7 *}\vspace*{.1cm}\\
\textsuperscript{1}Department of Genetics, Yale School of Medicine, New Haven, CT, USA\\
\textsuperscript{2}Department of Mathematics, Yale University, New Haven, CT, USA\\
\textsuperscript{3}Program for Applied Mathematics, Yale University, New Haven, CT, USA\\
\textsuperscript{4}Department of Physics, Yale University, New Haven, CT, USA\\
\textsuperscript{5}Institute of AI for Health, Helmholtz Pioneer Campus, Munich, Germany\\
\textsuperscript{6}Department of Computer Science, Yale University, New Haven, CT, USA\\
\textsuperscript{7}Program for Computational Biology and Bioinformatics, Yale University, New Haven, CT, USA\vspace*{.1cm}\\
\textdagger co-first authors, *co-senior authors\\
\textit{Correspondence:} \texttt{smita.krishnaswamy@yale.edu} \\
}

\begin{document}
\maketitle

\begin{abstract}
We introduce a new intrinsic measure of local curvature on point-cloud data called {\em diffusion curvature}. Our measure uses the framework of diffusion maps, including the data diffusion operator, to structure point cloud data and define local curvature based on the laziness of a random walk starting at a point or region of the data. We show that this laziness directly relates to volume comparison results from Riemannian geometry. We then extend this scalar curvature notion to an entire quadratic form using neural network estimations based on the diffusion map of point-cloud data. We show applications of both estimations on toy data, single-cell data and on estimating local Hessian matrices of neural network loss landscapes.
\end{abstract}


\section{Introduction}
\label{sec:intro}

With the advent of high-throughput high-dimensional point-cloud data in many fields including biomedicine, social science, physics, and finance, there is an increasing need for methods to extract structure and patterns. A key framework for exploring such data is the manifold hypothesis, which states that high-dimensional data arise as samples from an underlying lower dimensional manifold. The diffusion geometry and diffusion map framework first introduced by Coifman et al.~\cite{coifman2006diffusion} has proven to be a useful framework for the shape and structure of the data. Diffusion geometry involves converting data into an affinity graph which is then Markov normalized to form a data diffusion operator. This diffusion operator allows for the understanding of geometric features of the data such as manifold-intrinsic distances, and intrinsic manifold eigendimensions~(also known as diffusion maps). However, curvature, a salient feature of Riemannian geometry, has received less attention. Being an inherently smooth quantity, its extension to the discrete case is not straightforward. The most authoritative definitions of curvature in the discrete setting have been that of Ollivier--Ricci~\cite{sia_ollivier-ricci_2019} and of Forman~\cite{Forman03}, which compute curvature on edges of a graph. By contrast, here we propose a method of estimating curvature at {\em points or regions} in point-cloud data, where there are no inherent edges to consider---one of the first such formulations to our knowledge. 

To achieve this, we utilize the connection between data diffusion and volume to define a new formulation of curvature known as {\em diffusion curvature} based on the volume growth of graph diffusion starting from a vertex. Starting from a unit mass placed at one or several points, we show that in cases of positive curvature the volume growth tends to be slower with more of the mass remaining in the original point, whereas in cases of negative curvature the mass tends to disperse faster to other points. Therefore the ratio of mass in the original point, compared to the remainder of the cloud gives a scalar value of curvature. We note that this method avoids the complexity of having to consider all (potentially exponentially many) edges that would arise by simply attempting to directly translate notions of edge curvature to the data affinity graph.  In addition to the scalar notion of curvature, we also show that we are able to compute a quadratic form of arbitrarily high dimensions using neural network training based on diffusion probabilities as the input. 

We validate our notion of curvature on toy datasets as well as real single-cell point cloud data of very high dimensionality. Then, we study both the scalar diffusion as well as the diffusion based quadratic form in studying the area surrounding found minima in neural network loss landscapes. The Hessian of the parameter space is generally difficult to compute as it contains $\mathcal{O}(n^2)$ entries for a neural network with $n$ parameters. However, restricting ourselves to the area around the minima yields a lower dimensional space, where only a subset of parameter combinations change in the vicinity. We sample in this lower dimensional space and compute a diffusion operator which allows us to assess the curvature of the minima. First, we can assess if the point of convergence is indeed a minimum or potentially a saddle. Further, studies have shown that flat minima, i.e., minima situated in a neighborhood of the loss landscape with roughly similar error, tend to generalize better than sharp minima. These qualities can be assessed from curvature estimations of minima regions. 

Our {\em key contributions} include:

\begin{compactitem}
    \item The formulation of diffusion curvature, an intrinsic scalar valued curvature for regions of high dimensional data sampled from an underlying manifold, computed from data diffusion.
    \item Proof that our notion of curvature will result in higher values in positive curvature based on volume comparison results from Riemannian geometry.
    \item A neural network version that computes a quadratic form starting from a diffusion map.
    \item Validation of our framework on toy and single cell datasets.
    \item Application to Hessian estimation in neural networks.
\end{compactitem}


\section{Preliminaries and Background}
\label{sec:bg}

\subsection{High Dimensional Point Clouds and the Manifold Assumption}
\label{sec:bg-manifold-assumption}

A useful assumption in representation learning is that high-dimensional data is sampled from an intrinsic low-dimensional manifold that is mapped via nonlinear functions to observable high-dimensional measurements; this is commonly referred to as the manifold assumption. Formally, let $\mathcal{M}^d$ be a hidden $d$ dimensional manifold that is only observable via a collection of $n \gg d$ nonlinear functions $f_1,\ldots,f_n\colon \mathcal{M}^d \to \mathbbm{R}$ that enable its immersion in a high-dimensional ambient space as $F(\mathcal{M}^d) = \{\mathbf{f}(z) = (f_1(z),\ldots,f_n(z))^T : z \in \mathcal{M}^d \} \subset \mathbbm{R}^n$ from which data is collected. Conversely, given a dataset $X = \{x_1, \ldots, x_N\} \subset \mathbbm{R}^n$ of high-dimensional observations, manifold learning methods assume data points originate from a sampling $Z = \{z_i\}_{i=1}^N \subset \mathcal{M}^d$ of the underlying manifold via $x_i = \mathbf{f}(z_i)$, and aim to learn a low dimensional intrinsic representation that approximates the manifold geometry of~$\mathcal{M}^d$.

\subsubsection{Diffusion Maps}
\label{sec:bg-diff-map}

To learn a manifold geometry from collected point cloud data, we use the popular diffusion maps construction~\cite{coifman2006diffusion}. This construction starts by considering local similarities defined via a kernel $\mathcal{K}(x,y)$, with $x,y \in F (\mathcal{M}^d),$ that captures local neighborhoods in the data. While the Gaussian kernel is a popular choice for $\mathcal{K}$, it encodes sampling density information in its computation.
Hence, to construct a diffusion geometry that is robust to sampling density variations, we use an anisotropic kernel
$$\mathcal{K}(x,y) = \frac{\mathcal{G}(x,y)}{\|\mathcal{G}(x,\cdot)\|_1^{\alpha} \|\mathcal{G}(y,\cdot)\|_1^{\alpha}}$$
with $ \mathcal{G}(x,y) = e^{-\nicefrac{\|x-y\|^2}{\sigma}}$ as proposed in \cite{coifman2006diffusion}, where $0 \leq \alpha \leq 1$ controls the separation of geometry from density, with $\alpha = 0$ yielding the classic Gaussian kernel, and $\alpha = 1$ completely removing density and providing a geometric equivalent to uniform sampling of the underlying manifold. In \cite{coifman2006diffusion}, it is shown that in the limit of infinitely many points this becomes equivalent to a Laplace--Beltrami operator, which encodes geometric properties of the underlying manifold. 

Next, the similarities encoded by $\mathcal{K}$ are normalized to define transition probabilities $p(x,y) = \nicefrac{\mathcal{K}(x,y)}{\|\mathcal{K}(x,\cdot)\|_1}$ that are organized in an $N \times N$ row stochastic matrix $\mathbf{P}_{ij} = p(x_i, x_j)$ that describes a Markovian diffusion process over the intrinsic geometry of the data.

Finally, a diffusion map is defined by taking the eigenvalues $1 = \lambda_1 \geq \lambda_2 \geq \cdots \geq \lambda_N$ and (corresponding) eigenvectors $\{\phi_j\}_{j=1}^N$ of $\mathbf{P}$, and mapping each data point $x_i \in {X}$ to an $N$ dimensional vector:
\begin{equation}\label{equation:diffusionmap}
\Phi_t(x_i) = [\lambda_1^t \phi_1(x_i),\ldots,\lambda_N^t \phi_N(x_i)]^T
\end{equation}
Here $t$ represents a diffusion-time. In general, as $t$ increases, most of the eigenvalues become negligible;  truncated diffusion map coordinates can thus be used for dimensionality reduction, and Euclidean distances in this space are a manifold-intrinsic distance~\cite{coifman2006diffusion, moon_visualizing_2019}.

The eigenvalues of the Laplace-Beltrami operator, as well as the eigenvalues of the diffusion map, encode geometry information about a compact Riemannian manifold~\citep{coifman2006diffusion}, as can be seen via the asymptotic expansion of the trace of the heat kernel~\citep{gordon}. Dimension, volume, and total scalar curvature are spectrally determined properties of the manifold. Given the relationship between diffusion coordinates and the eigenvectors of the Laplace-Beltrami operator, one can surmise that the diffusion operator also encodes geometric information. These relations motivate our use of diffusion to measure curvature on a data manifold. Supplementary materials contain additional details in the Riemannian setting. 

\subsection{Discrete Curvature in Riemannian Geometry}
\label{sec:bg-riemann-geometry}

There is a long history in Riemannian geometry of using curvature to study geometric and topological properties of a Riemannian manifold. In particular, lower bounds on Ricci curvature have been related to diameter (Bonnet-Myers \cite{myers}), volume (Bishop-Gromov \cite{bishop}), the Laplacian (Cheng-Yau \cite{Yau75}), the isoperimetric constant (Buser \cite{Buser1982}), and topological properties of the manifold (Hamilton's Ricci flow \cite{hamilton}). Although somewhat paradoxical, there has recently been work to extend these smooth Riemannian ideas to the discrete setting, in particular to graphs and Markov chains. Especially notable in this regard is the definition of discrete Ricci curvature due to Ollivier which makes use of the transport distance between probability distributions. 

Ollivier's Ricci curvature \cite{sia_ollivier-ricci_2019} starts with a metric space $X$ equipped with a random walk (a probability measure $m_x(\cdot)$ for each $x \in X$) and assigns the edge-wise scalar curvature
$k(x,y)=1-\nicefrac{W_1(m_x,m_y)}{d(x,y)}$ where $W_1(\cdot,\cdot)$ is the $L^1$ transportation distance~(also known as earth mover's distance or Wasserstein distance). In the Riemannian setting, if one defines the random walk to be $dm^r_x(y)=\nicefrac{dvol(y)}{vol B(x,r)}$ then Ollivier demonstrates that $k(x,y)=\nicefrac{r^2Ric(v,v)}{2d+2}+O(r^3+d(x,y)r^2)$ where $v$ is a unit tangent vector, and $y$ is a point on the geodesic issuing from $v$ with $d(x,y)$ small enough. Ollivier (and others \cite{seong}, \cite{Yau2011}) have used lower bounds on this notion of curvature to study global properties of the space (i.e.~diameter, volume growth, spectral gap). 

The main downsides of Ollivier's Ricci curvature are that it is an edge-wise notion and that it requires knowledge of the transport distance. A priori point cloud data does not come with a graph structure. We therefore propose a diffusion based method for prescribing a point-wise scalar curvature to point cloud data. Like Ollivier's notion, this \emph{diffusion curvature} is based on the idea that the spread of geodesics is influenced by Riemannian curvature. We prove theoretical bounds on the diffusion curvature by appealing to the Bishop-Gromov volume comparison theorem.

\subsection{The Loss Landscape and Hessian of Neural Networks}
\label{sec:bg-hessian}

Neural networks are trained such that their output function $f(X, \theta)$ matches a given function $y(X)$ on training data $X$ by performing, e.g., stochastic gradient descent on a loss function comparing $y$ and $f$ with respect to the parameters $\theta$. A typical example of a loss function is mean squared error $\mathcal{L}(X,\theta)=\sum_{x \in X}||f(x)-y(x)||_2$. 

The loss function defines a loss landscape for fixed training data $X$. Previous studies have sought to directly examine the loss landscape of a neural network. Li et al.~\cite{li_visualizing_2018} used random samples around found minima to visualize and quantify sharpness of minima, Horoi et al.~\cite{horoi_exploring_2022} utilized jump-and-retrain sampling to visualize and classify ``good'' and ``bad'' minima. Here, we use data diffusion to directly estimate the Hessian of the loss landscape, with respect to the parameters of a neural network,  around found minima, providing a quantitative measure of the region in terms of its spectrum and condition number.

In general, the Hessian of a scalar-valued function, $\mathcal{L}(\mathbf{\theta})$ is a symmetric matrix of partial derivatives with the quadratic form:

$$H\mathcal{L}(\mathbf{v}) = 
    \begin{pmatrix} v_1 & \dots & v_n \end{pmatrix}
    \begin{pmatrix} 
    \frac{\partial^2 \mathcal{L}}{\partial \theta_1 \partial \theta_1} & \dots & \frac{\partial^2 \mathcal{L}}{\partial \theta_1 \partial \theta_n} \\
    \vdots & \ddots & \vdots\\
    \frac{\partial^2 \mathcal{L}}{\partial \theta_n \partial \theta_1} & \dots & \frac{\partial^2 \mathcal{L}}{\partial \theta_n \partial \theta_n} 
    \end{pmatrix} 
    \begin{pmatrix}v_1 \\ \vdots \\ v_n \end{pmatrix}
$$

Since the gradient is zero at critical points, i.e., $\nabla \mathcal{L}(\mathbf{\theta_0}) = 0$, the quadratic approximation to the function around its critical points is given by:

$ \mathcal{L}(\mathbf{\theta}) \approx \mathcal{L}(\mathbf{\theta_0}) + \nicefrac{1}{2}H\mathcal{L}(\mathbf{\theta}-\mathbf{\theta_0})$.
We see that the signature of the Hessian (the signs of its eigenvalues) precisely classifies the critical points ($\nabla \mathcal{L}(\mathbf{\theta_0}) = 0)$ as a local maximum~(all eigenvalues are negative), minimum~(all eigenvalues are positive) or saddle~(eigenvalues have mixed signs). 

In the context of neural networks, second-order information about the loss function contains useful information about the generalizability of minima, the heuristic being that \emph{flat} minima generalize better than \emph{sharp} minima. However, computing the full Hessian of the parameter space is intractable for high-dimensional systems (such as the parameter space of a neural network). To circumvent this, we estimate a low dimensional Hessian of the loss function around its optimum using a quadratic approximation. Restricting ourselves to the area around the optimum yields a lower dimensional space, where only a subset of parameter combinations change in the vicinity.
This Hessian approximation is achieved by sampling around the optimum, constructing a diffusion operator based on these sampled points, and using a neural network to learn the coefficients of the quadratic approximation.

\section{Methods}
\label{sec:methods}

\subsection{Diffusion Curvature}
\label{sec:methods-diffusion-curvature}

We now illustrate how diffusion can be used to measure the relative spreading of geodesics under the influence of Riemannian curvature. 
To build intuition, we first discuss the case of surfaces~(dimension $n = 2$). Three canonical surfaces are the sphere~(or surface of a ball), the cylinder, and the saddle. The Gaussian curvature of these surfaces are positive, zero, and negative, respectively. Imagine taking a sticker and trying to adhere it to one of these surfaces. The sticker, having been printed on a flat piece of paper, has zero Gaussian curvature. It will adhere perfectly to the cylinder, will bunch up~(there will be too much sticker material) when trying to adhere it to the sphere, and will rip~(there will be too little sticker material) when trying to adhere it to the saddle. 

This example showcases the area comparison definition of Gaussian curvature~\cite[pp.\ 225--226]{Berger03}, where one computes the limiting difference between the area $A(r)$ of a geodesic disk on the manifold and a standard Euclidean disk,  $$K=\lim\limits_{r \to 0^+} 12\frac{\pi r^2 - A(r)}{\pi r^4}.$$ For example, Gaussian curvature will be positive when $\pi r^2 > A(r)$, i.e., when the area of the sticker exceeds the area of the corresponding geodesic disk on the sphere. The Bishop-Gromov volume comparison theorem formalizes how to extend this into dimensions $n>2$.

\begin{theorem}[Bishop-Gromov]
\label{thm:bishop-gromov}
Let $(M^n,g)$ be a complete Riemannian manifold with Ricci curvature bounded below by $(n-1)K$. Let $B(p,r)$ denote the ball of radius $r$ about a point $p \in M$ and $B(p_K,r)$ denote a ball of radius $r$ about a point $p_K$ in the complete $n$-dimensional simply-connected space of constant sectional curvature $K$. Then $\phi(r)=\nicefrac{\text{Vol}(B(p,r))}{\text{Vol}(B(p_K,r))}$ is a non-increasing function on $(0,\infty)$ which tends to $1$ as $r \to 0$. In particular, $\text{Vol}(B(p,r)) \leq \text{Vol}(B(p_K,r)).$
\end{theorem}

Theorem \ref{thm:bishop-gromov} captures the sticker phenomenon: as curvature increases, the volume of comparable geodesic balls decreases. Positive curvature corresponds to geodesic convergence and smaller volumes, whereas negative curvature corresponds to geodesic divergence (spread) and larger volumes.

The discrete nature of the data manifold makes it impossible to compare volumes as the distance scale (radius) goes to zero. Theorem \ref{thm:coifman} gives us access to the intrinsic distance between data points, i.e., the diffusion distance,  which we then use to define metric balls on the Riemannian manifold $M$. 

\begin{theorem}[Coifman et al. \cite{coifman_geometric_2005}]
\label{thm:coifman}
The diffusion map $\Phi_t(x_i) = [\lambda_1^t \phi_1(x_i),\ldots,\lambda_N^t \phi_N(x_i)]^T$ embeds data into a Euclidean space where the Euclidean distance is equal to the diffusion distance $D_m$, i.e. $D_m^2(x,y) = \| \Phi_t(x) - \Phi_t(y) \|^2 (1+O(e^{-\alpha m}))$ .
\end{theorem}

We can now define $B_m(x,r)=\{ y \in M \colon D_m(x,y) \leq r \} \subset M$ to be the ball centered at $x \in M$ with diffusion radius $r$. Let $B(x,r)$ denote the set of sampled points from $B_m(x,r)$ and let $|B(x,r)|$ denote the cardinality of this set~(we set aside the problem of choosing an appropriate~$r$ for now; see the supplementary materials for details). To define diffusion curvature, we generate a random walk from a point $x$ by diffusing a Dirac based at $x$, i.e., $m_x(\cdot)=\delta_x{\bf P}^t$ where $\delta_x$ is the one-hot vector, so that $m_x(y)=\mathbf{P}^t(x,y) $ is the transition probability from $x$ to $y$. 

\begin{definition}
The pointwise \emph{Diffusion Curvature} $C(x)$ is the average probability that a random walk starting from a point $x$ ends within $B(x,r)$ after $t$ steps of data diffusion, i.e.,
\begin{equation}
\label{eqn:diffusion_curvature}
C(x) = \frac{\sum_{y \in B(x,r)}m_x(y)}{|B(x,r)|}
\end{equation}
\end{definition}

We can extend this definition to a contiguous region $U$ of the manifold $M$ consisting of neighboring points $U=\{x_j\}_{j=1}^k $ by defining $m_U(\cdot)=\delta_U {\bf P}^t$, where $\delta_U$ is the indicator vector on the set $U$.

We have $C(x) \in [0,1/N]$ where $N=|B(x,r)|$ with larger values indicating higher curvature relative to lower values. The idea is to use diffusion probabilities to capture the relative spreading of geodesics. Intuitively, a random walker is more likely to return to their starting point in a region of positive curvature (where paths converge) than negative curvature (where paths diverge). 

More precisely, in negatively curved regions, the random walker can get lost in the various divergent (disconnected) branches, whereas in positively curved regions the paths of the random walker exhibit more inter-connectivity, so that the return probability of a walk is higher. We make this idea more formal in the following:

\begin{theorem}
If we sample uniformly from a Riemannian manifold with $Ric(M^n) \geq k(n-1)$ then $C(x) \geq C(x_k)$, where $x_k$ is a point from a uniform sampling of the complete $n$-dimensional simply-connected space of constant sectional curvature $k$.
\end{theorem}

\begin{proof}
From Theorem~\ref{thm:bishop-gromov} we have $vol(B_m(x,r)) \leq vol(B_m(x_k,r))$ and  together with uniform sampling this yields $|B(x,r)| \leq |B(x_k,r)|$. We  then have $1/|B(x,r)| \geq 1/|B(x_k,r)|$ so that even with uniform transition probability measures we achieve the desired inequality.
Moreover, as the probability measure $m_x(\cdot)$ is constructed by diffusing a Dirac according to affinities based on a Gaussian kernel, the measure decays with relative diffusion distance from $x$. Hence for sets centered at $x$, larger cardinality implies lower average transition probability. We may therefore conclude that

$$C(x) = \frac{\sum_{y \in B(x,r)}m_x(y)}{|B(x,r)|} \geq
 \frac{\sum_{y \in B(x_k,r)}m_{x_k}(y)}{|B(x_k,r)|} = C(x_k).$$
\end{proof}

\begin{figure}
    \centering
    \includegraphics[width=0.95\linewidth]{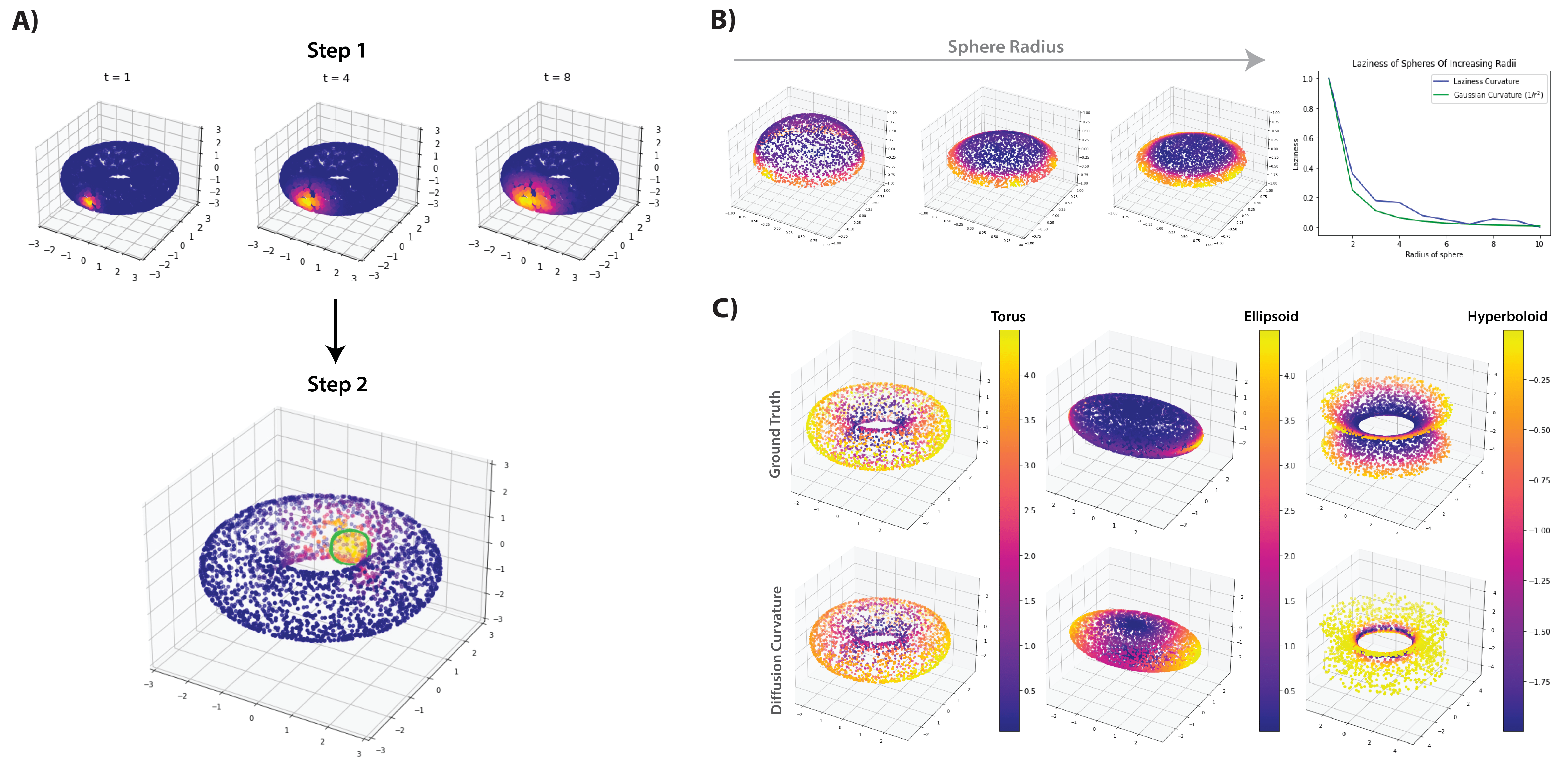}
    \caption{Diffusion recovers Gaussian curvature information (up to a scaling factor and offset) on test manifolds, e.g. sphere, torus, ellipsoid and hyperboloid.}
    \label{fig:diff-curvature}
\end{figure}

\paragraph{Limitations.} While diffusion curvature can be computed easily, it also necessitates an accurate sampling procedure, which can be problematic for difficult spaces. Moreover, due to its formulation, diffusion curvature only provides relative~(non-negative) values.

\subsection{Neural Network Implementation of Quadratic Forms}
\label{sec:methods-nn}

In the previous section, we have shown that the diffusion operator and associated diffusion probabilities contain information for computation of curvature from point cloud data. In this section, we extend this measure from a scalar quantity to an entire quadratic form by training a neural network on the diffusion map embedding of the point cloud. A schematic of our neural network, called CurveNet, is shown in Figure~\ref{fig:nn_schematic}.

CurveNet takes as input points $\{X_i = (\mathbf{x}_i, y_i)\}_{i=1}^N$ sampled near a local minimum or saddle point $(\mathbf{x}_s, y_s)$. Here $\mathbf{x}_i \in \mathbb{R}^k$ are input parameters for an objective function or loss, $f$, with $y_i = f(\mathbf{x}_i)$ and $\nabla f(\mathbf{x}_s) = 0$. We map each data point, $X_i$, to a $N$ dimensional diffusion vector, $\Phi_t(X_i) = [\lambda_1^t \phi_1(X_i),\ldots,\lambda_N^t \phi_{n}(X_i)]^T$.

Recall from Equation \ref{equation:diffusionmap} that the diffusion map is an eigendecomposition of the diffusion operator $\mathbf{P}^t$. Hence, without reduction of dimensionality, it encodes all information present in the diffusion operator, and contains $N$ dimensions for a sampling of $N$ points. Here we utilize this to estimate a $k \times k$ quadratic form of the sampled point cloud, $y_i=\mathbf{x}_i^T\hat{Q}\mathbf{x}_i$, using a neural network. Note that when the point cloud is a sample of a local region in the manifold, this gives a localized quadratic form describing 2\textsuperscript{nd} order information in the region sampled.

Since our focus is on local quadratic approximation, we train CurveNet on samples from idealized surfaces of the form

$f(x) = x^TQx$

where $Q$ is a symmetric matrix with dimensionality dependent on the cardinality of the point cloud; see Figure~\ref{fig:toy-data} for some examples of generated $2$-D surfaces.

We generate a series of such surfaces by varying parameters within $Q$ and sampling points. We then provide the diffusion map embedding of the sampled points to the neural network and train it to predict $Q$ using a mean-squared error loss. We also include an $L_1$ regularization term to promote sparsity in the coefficients, which estimates the dimensionality of the input data, leading to

\begin{equation}
\mathcal{L}(Q,\hat{Q}) = \sum_{j=1}^{k} \sum_{j'=1}^{j} (Q_{jj'} - \hat{Q}_{jj'})^2 + \alpha |Q_{jj'}|.
\end{equation}

\begin{figure}
    \centering
    \includegraphics[width=0.85\linewidth]{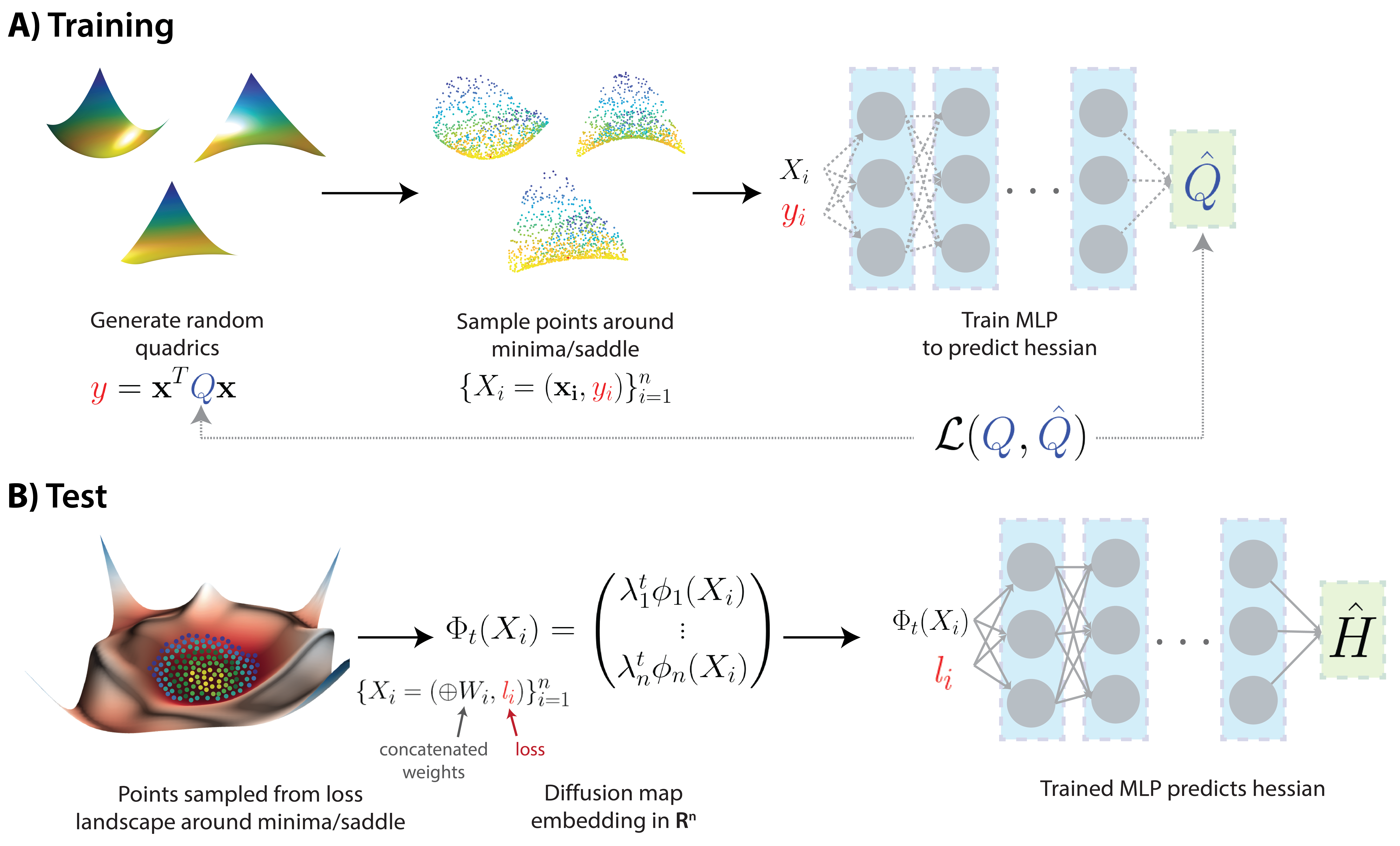}
    \caption[Schematic of the neural network for quadric approximation.]{The neural network is trained on randomly generated quadrics and tested on points sampled near local minima or saddles in the loss landscape of neural networks.}
    \label{fig:nn_schematic}
\end{figure}

\paragraph{Curvature and the Hessian.} In the classical setting of surfaces, a relevant notion of curvature is the Gaussian curvature $\kappa$. In the special case that $M$ is realized as the graph of a function $f(x, y)$, the Gaussian curvature can be calculated explicitly from the Hessian matrix $H$ of $f$ by the formula $$\kappa(x, y) = \frac{\text{det}(H_{(x, y)}f)}{(1+\abs{\nabla_{(x, y)} f}^2)^2}.$$ 

Note that at a local optimum of $f$ (where $\nabla f = 0$) the Gaussian curvature is completely characterized by the eigenvalues of the Hessian.

In higher dimensions, by contrast, a single number is no longer sufficient to capture the full curvature information of a Riemannian manifold, and one instead appeals to the Riemann curvature tensor. In the special case that $M$ is realized as the graph of a function $f \colon \mathbb{R}^n \to \mathbb{R}$, the spectrum of the Hessian of $f$ again contains important~(albeit not complete) curvature information. As we are interested in measuring curvature in the setting of loss landscapes, we apply CurveNet described above to directly compute the Hessian of the parameters with respect to the loss function of a neural network in question. 

\section{Results}
\label{sec:results}

We validate both diffusion curvature and CurveNet using 

\begin{inparaenum}[(a)]
    \item synthetic test cases, and
    \item single-cell RNA sequencing data~(scRNA-seq).
\end{inparaenum}

Moreover, we analyze the quality of our curvature estimation methods in a data sampling scenario, where we sample points in the vicinity of an optimum. All test cases are chosen to assess the efficiency and accuracy of our methods. 

We trained CurveNet on $N=1000$ samples from idealized surfaces, since this defines the dimensions of the diffusion map used for curvature estimation on loss landscapes. As the parameter space of neural networks can be large, sampling points around the minima is a method of reducing dimensionality and restricting the estimation to a less complex space. We used intrinsic dimensions of $k=\{2, 3, \ldots 20\}$ such that the same neural network can learn higher or lower dimensional spaces. We trained on $5000$ different randomly generated quadrics, all sampled using $N=1000$ points, for each intrinsic dimension. CurveNet then outputs a quadratic form with $K(K+1)/2$ entries, where ($K = \max(k) = 20$), which we compare with the known ground truth using mean squared error (Table \ref{tab:toy-data}). Training and testing were done on $8$ core Tesla K80 GPUs with $24$ GB memory/chip. Architectural details and code for Curvenet and the implementation of diffusion curvature are available via an anonymized GitHub URL provided in the supplementary material. 

\begin{table}
  \caption{MSE ($\mu \pm \sigma$) of neural network Hessian estimation. All values scaled by $10^3$, lower is better.}
  \label{tab:toy-data}
  \centering
  \small
  \begin{tabular}{lcccccc}
    \toprule
    \textbf{Test Data} & \multicolumn{1}{c}{CurveNet} & \multicolumn{1}{c}{CurveNet } & \multicolumn{1}{c}{Baseline}\\
    Intrinsic dim. & (\# epochs = 1000) & (\# epochs = 100) & (random)  \\
    (\# nonzero coeffs.) &\\
    \midrule
    2 (4)    & $2.115\pm1.250$ & $14.071\pm8.664$ & $612.145\pm2.836$\\
    5 (16)   & $0.964\pm0.144$ & $4.865\pm3.382$ & $554.652\pm1.704$\\
    10 (56)  & $0.752\pm0.013$ & $0.798\pm0.029$ & $530.203\pm0.645$\\
    15 (121) & $0.766\pm0.022$ & $0.085\pm0.007$ & $521.050\pm0.485$\\
    20 (211) & $0.712\pm0.034$ & $0.082\pm0.004$ & $515.739\pm0.397$\\
    \bottomrule
  \end{tabular}
\end{table}

\begin{figure}
    \centering
    \includegraphics[width=0.75\linewidth]{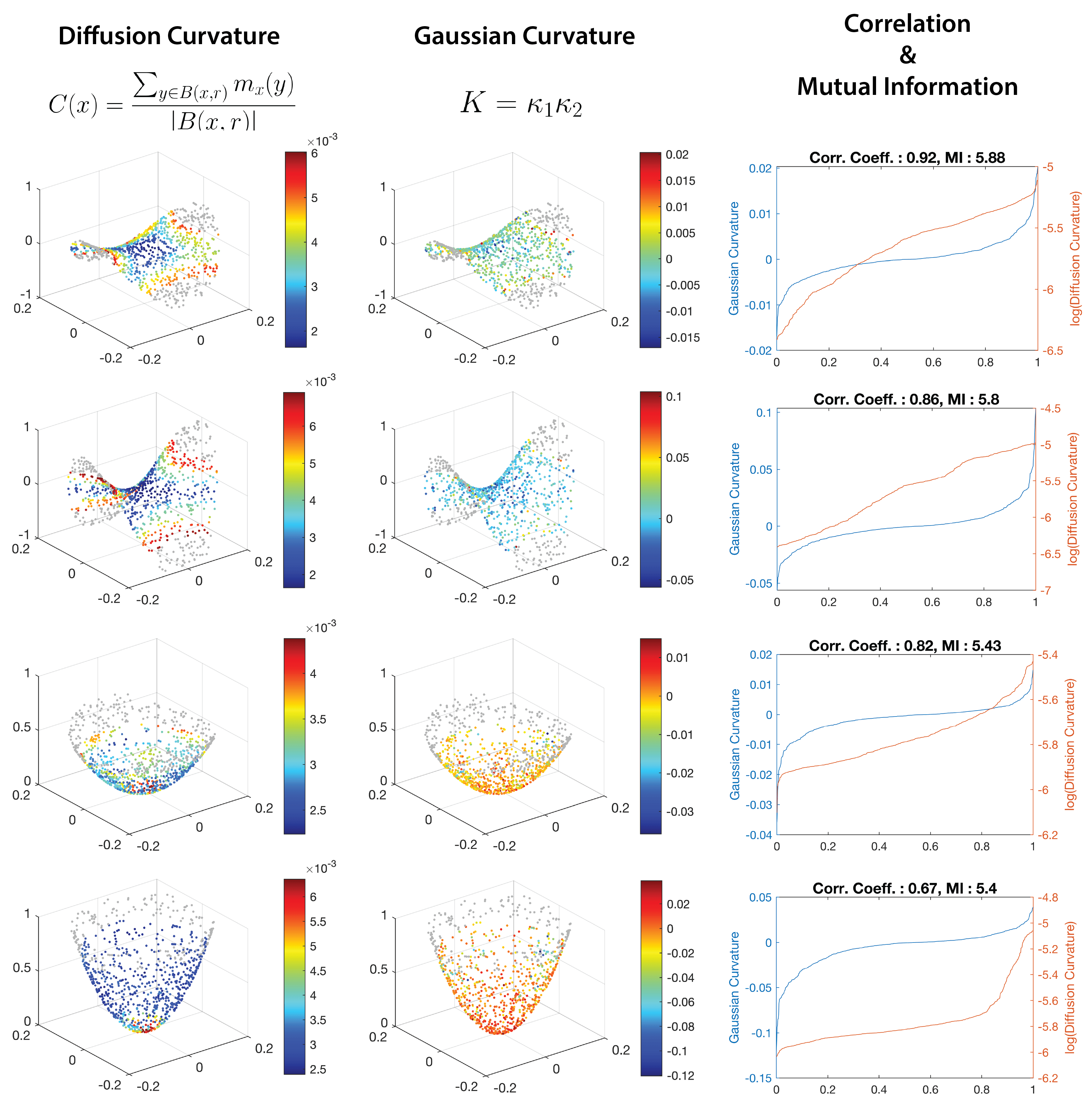}
    \caption{Curvature estimation on toy dataset of quadric surfaces.}
    \label{fig:toy-data}
\end{figure}

\subsection{Toy test cases for curvature estimation}
\label{sec:results-toy-data}

We generated a series of synthetic datasets where the primary objective is to estimate curvature at central points (which are not affected by edge effects). Figure \ref{fig:toy-data} shows a series of artificially generated $3D$ surfaces whose curvature varies in two principal directions from positive to negative. In the left column we show the curvature estimate given by Diffusion Curvature (Eqn. \ref{eqn:diffusion_curvature}). The second column depicts Gaussian curvature which is the product of the two principal curvatures at each point, the third column contains biaxial plots showing the correlation between Gaussian and Diffusion curvature.

We observe that Diffusion Curvature is capable of approximating/capturing Gaussian Curvature: despite being scaled differently, Diffusion Curvature highlights essentially the same structures as Gaussian Curvature. This qualitative observation is quantified by the correlation. Overall, we obtain high correlations for the different surfaces. The last surface is slightly different, as high values measured by Diffusion Curvature are more concentrated within the ``cusp'', whereas Gaussian curvature spreads out such values over a larger part of the surface. This example demonstrates the benefits of Diffusion Curvature, though: in the context of loss landscape analysis, Diffusion Curvature is much more sensitive to such sharper minima, thus facilitating their detection.

\subsection{Curvature estimation for single-cell data}
\label{sec:results-single-cell}

We estimated the curvature of a publicly available single-cell point cloud dataset obtained using mass cytometry of mouse fibroblasts. Mass cytometry is used to quantitatively measure $2005$ mouse fibroblast cells induced to undergo reprogramming into stem cell state using using $33$ channels representing various protein biomarkers. Such a system is often called induced pluripotent stem cell (iPSC) reprogramming~\citep{zunder_continuous_2015}. This dataset shows a progression of the fibroblasts to a point of divergence where two lineages emerge, one lineage which successfully reprograms and another lineage that undergoes apoptosis~\citep{moon_visualizing_2019}. We note that our model correctly identifies the initial branching point (with cells that don't survive) as having low values of diffusion curvature indicating relatively negative curvature due to divergent paths out of the point (resulting in divergent random walks, see Figure~\ref{fig:ipscpred}). On the other hand it shows higher values indicating flat curvature along the horizontal branch. 

We also applied diffusion curvature to a single-cell RNA-sequencing dataset of human embryonic stem cells \cite{moon_visualizing_2019}. These cells were grown as embryoid bodies over a period of $27$ days, during which they start as human embryonic stem cells and differentiate into diverse cellular lineages including neural progenitors, cardiac progenitors, muscle progenitors, etc. This developmental process is visualized using PHATE in Figure \ref{fig:embryoid} (left), where embryonic cells (at days 0-3, annotated in blue) progressively branch into the two large splits of endoderm (upper split) and ectoderm (lower split around day 6. Then during days 12-27 they differentiate in a tree-like manor into a multitude of lineages.  Diffusion curvature, which is illustrated in the plot on the right shows that the tree-like structure that emerges during days 12-27 is consistently lower curvature than the initial trajectory which proceeds in a linear manner at days 0-3. This accords with the idea that divergent lineage structure is associated with low (negative) curvatures.  Conversely, the endpoints of the transition corresponding to the stem cell state (days 0-3) and differentiated state (days 18-27) are associated with relatively high diffusion curvature values, indicative of positive curvature.

\begin{figure}
    \centering
    \includegraphics[width=0.5\linewidth]{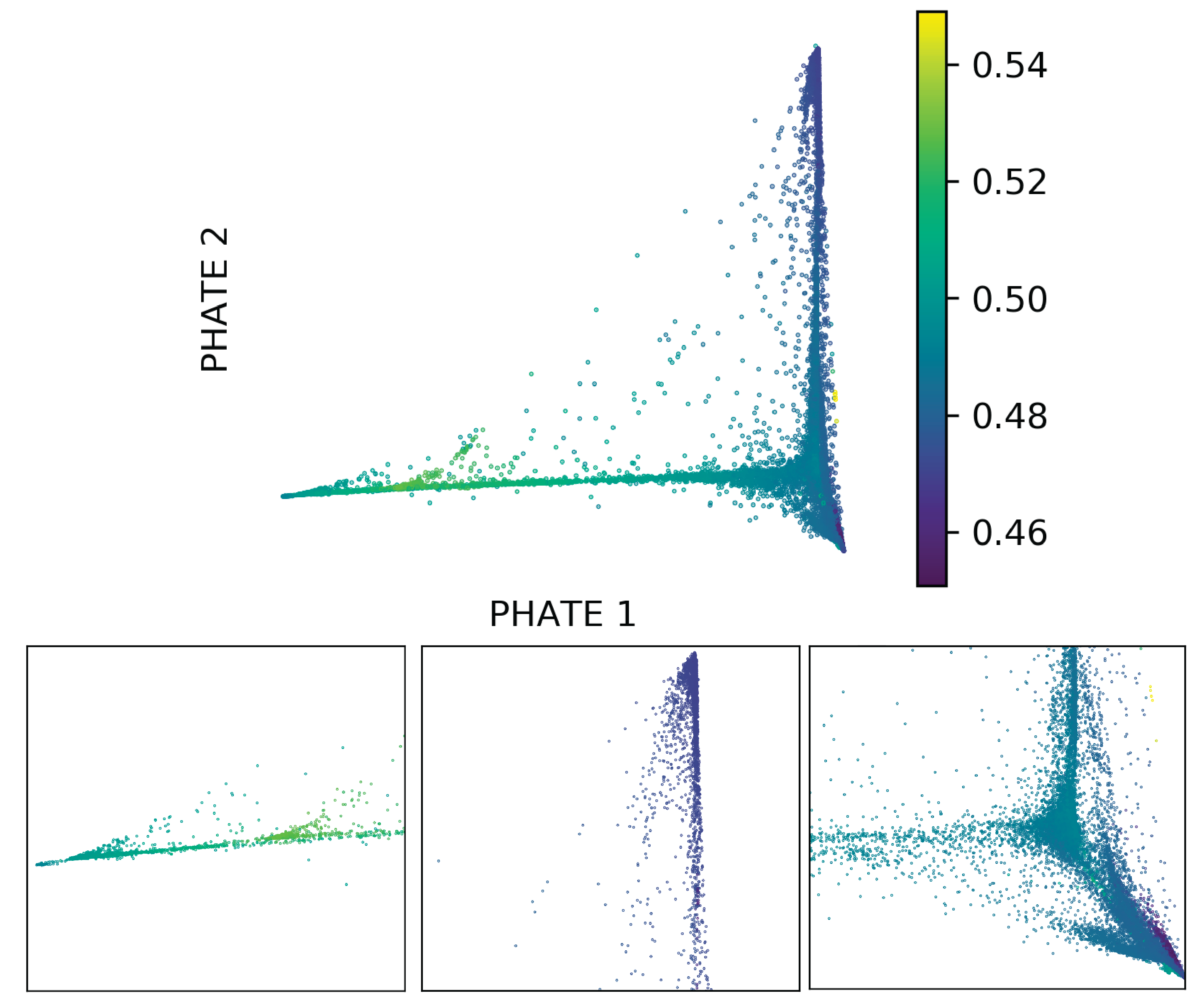}
    \caption{Diffusion curvature of the iPSC data.}
    \label{fig:ipscpred}
\end{figure}

\begin{figure}
    \centering
    \includegraphics[width=0.9\linewidth]{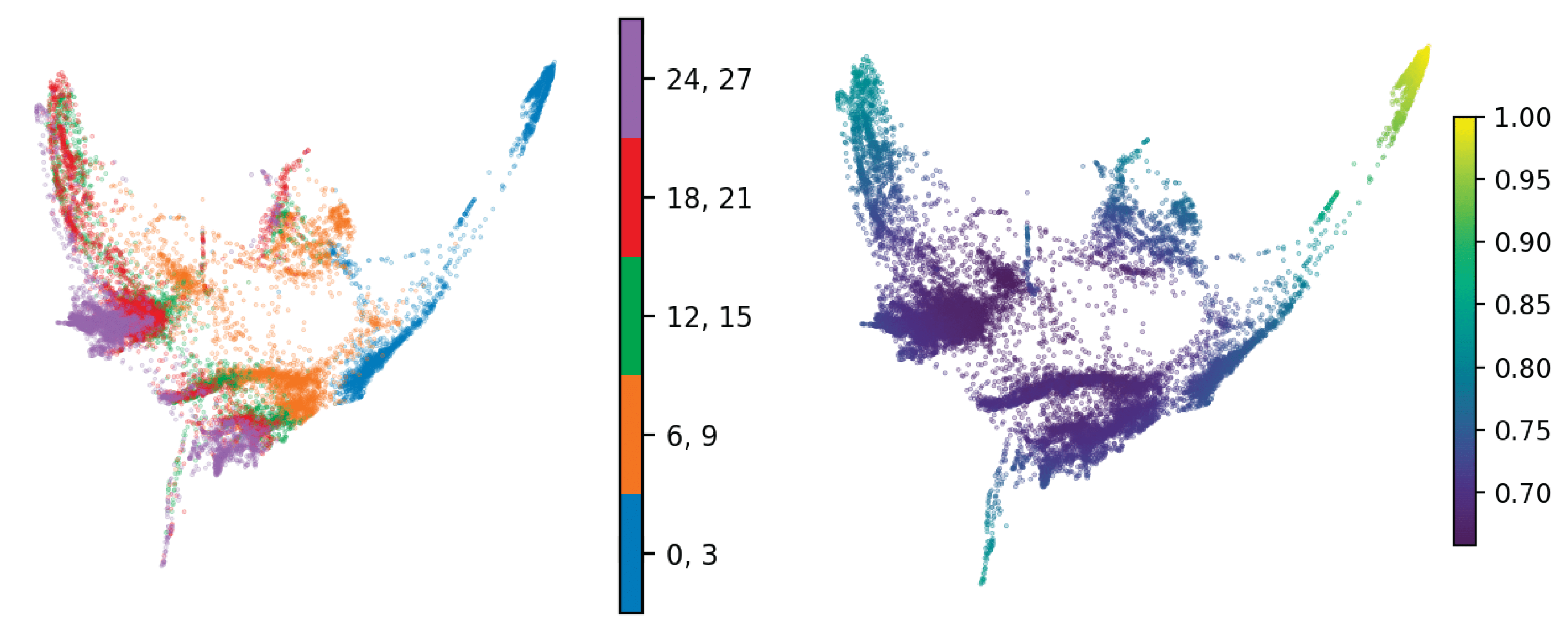}
    \caption{Diffusion curvature of embryonic stem cell differentiation. Left: PHATE visualization of scRNA-seq data color coded by time intervals. Right: PHATE plot colored by diffusion curvature values.}
    \label{fig:embryoid}
\end{figure}

\subsection{Hessian estimation and sampling}

To obtain second-order information around a critical point, $x_s$, we estimate a local Hessian: we first (for both toy test cases and neural networks) sample $n$ points, $x_1,x_2 \dots x_n$,  $x_i \in \mathbb{R}^k, $ where $k$ is the full dimensionality of the domain of the toy functions or the neural network: $f:\mathbbm{R}^k \rightarrow{} \R$ and the value of the objective function or loss, $f \in \R$ can be obtained by evaluating $f(x_i)$. In both settings, this allows us to obtain an $\mathbbm{R}^{n \times (k+1)} $ matrix. Each row represents a sampled point $x_i \in \mathbb{R}^k $ with its associated loss or objective function in the last column. We consider this column as a special loss axis and use it as an input to the neural network to learn the coefficients along with the diffusion axes, which are obtained from the sampled points. 

The diffusion axes were obtained by constructing a diffusion map in $\mathbbm{R}^{n}$ based on the sampled points around the optimum, i.e., the minimum. We uniformly sampled $1000$ points on a $k$-dimensional hypersphere which was then scaled by the parameter space of the optimum or saddle, as well as the gradient at that point. Special care was taken to ensure the points were sampled locally by evaluating at the relative difference between the evaluated loss at locally sampled points around the optimum or saddle and the actual loss at the optimum or saddle. We then use these diffusion coordinates~$\Phi(x)$ and the value at the sampled point~$f(x) \in \mathbbm{R}$ as an input pair $(\Phi(x),f(x))$ to the neural network, which we ultimately use to estimate the Hessian. 

Figures~\ref{fig:mnist-loss} and \ref{fig:mnist-cnn-loss} show the eigenspectrum of the Hessians estimated using CurveNet. We observe that the number of \emph{negative} eigenvalues of the Hessian matrix~(indicative of a maximum in the loss landscape) decrease over training epochs. In Figure~\ref{fig:mnist-loss}, when comparing epoch~$25$ and epoch~$200$, we observe a marked shift in the~(cumulative) density of eigenvalues towards positive eigenvalues, showing that the feed-forward neural network is approaching a minima in the parameter space. Similarly, in Figure~\ref{fig:mnist-cnn-loss} a convolutional neural network trained on MNIST approaches a sharp minima (large positive eigenvalues) over training epochs. These results are in agreement with our understanding of how model parameters are optimized during stochastic gradient descent, and demonstrate the capability of CurveNet to approximate to local Hessians. Future work will further validate this methodology using networks that are deliberately trained to produce poor minima and exhibit low generalizability.

\begin{figure}
    \centering
    \includegraphics[width=0.85\linewidth]{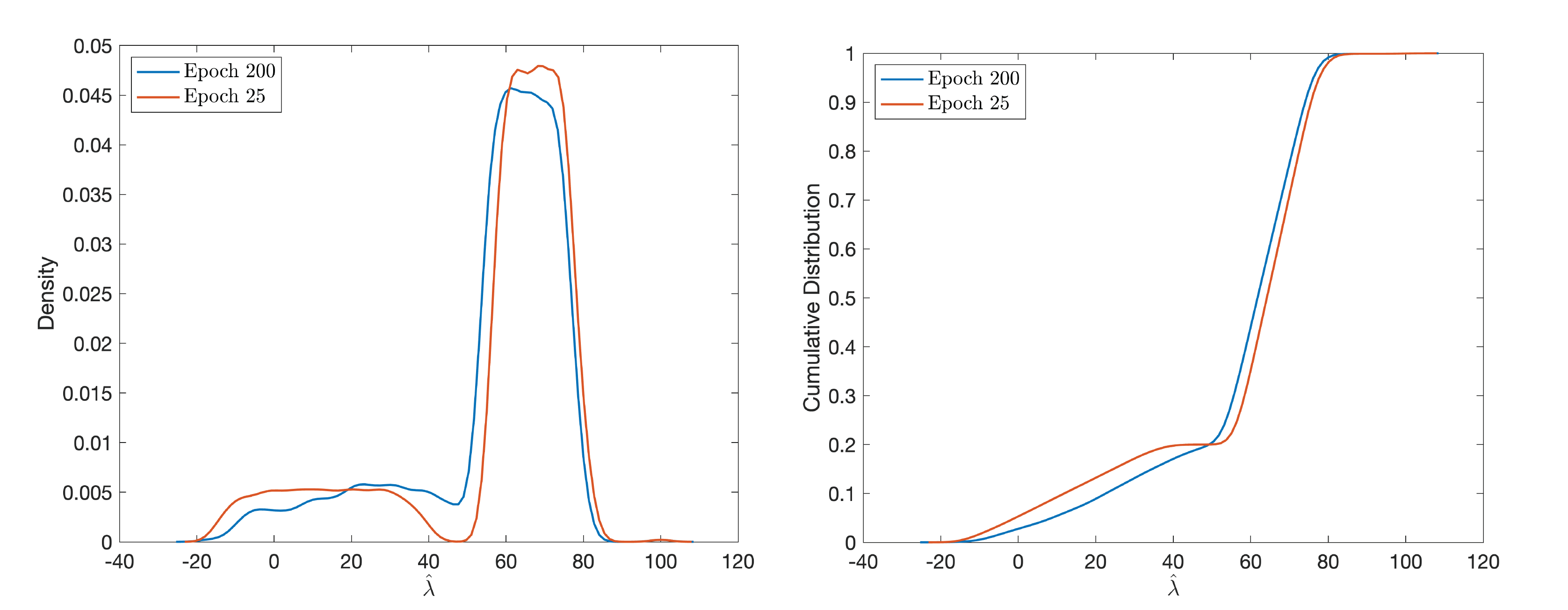}
    \caption{%
        Density~(left) and cumulative density~(right) of the eigenspectrum when sampling around an optimum in the loss landscape of a feed-forward network trained on MNIST. At epoch $200$, we observe a substantial decrease in the number of negative eigenvalues, indicating that the model parameters have shifted towards a better (sharper) minima.
    }
    \label{fig:mnist-loss}
\end{figure}

\begin{figure}
    \centering
    \includegraphics[width=0.85\linewidth]{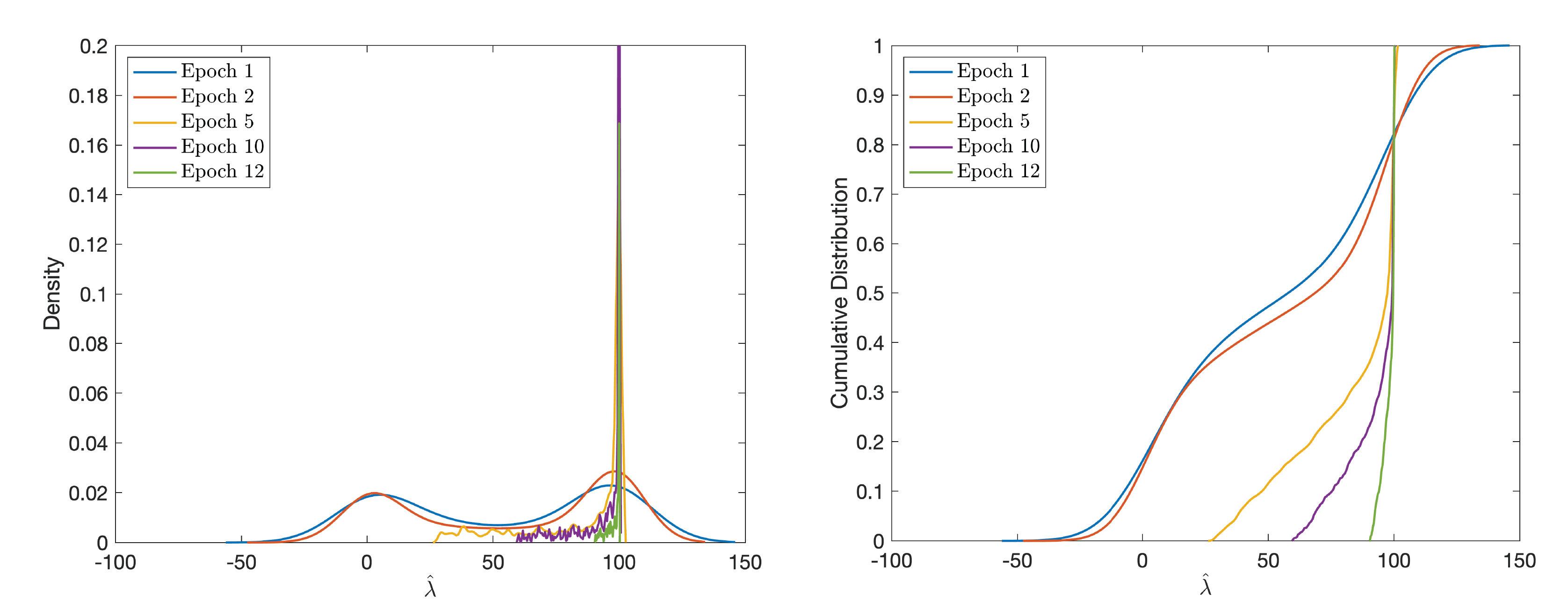}
    \caption{%
        Density~(left) and cumulative density~(right) of the eigenspectrum when sampling around an optimum in the loss landscape of a CNN trained on MNIST. At epoch $\ge 5$, we do not observe any negative eigenvalues, indicating that the model parameters have reached a local minima.
    }
    \label{fig:mnist-cnn-loss}
\end{figure}

\section{Conclusion}
\label{sec:conclusions}

We proposed \emph{diffusion curvature}, a new measure of intrinsic local curvature of point clouds. Our measure leverages the laziness of random walks, obtained using the diffusion maps framework. We link diffusion curvature to existing volume comparison results from Riemannian geometry. In contrast to these notions, diffusion curvature can be computed effectively via neural networks even for high-dimensional data. While we demonstrated the effectiveness of such a curvature measure by analyzing numerous datasets of varying complexities, our formulation also leads to new research directions. Of particular interest will be proving additional results about our measure, relating it to existing quantities such as the Laplace--Beltrami operator, as well as formally proving its stability properties. We also want to develop new methods that use diffusion curvature to compare different datasets; being a quantity that is invariant under transformations~(such as rotations of a point cloud), we consider diffusion curvature to be a suitable candidate for assessing the similarity of high-dimensional complex point clouds. 


\bibliographystyle{unsrt}  
\bibliography{template}
    
\end{document}